\documentclass[english,11pt]{article}

\usepackage{microtype}
\usepackage{graphicx}
\usepackage{subfigure}
\usepackage{booktabs} 

\usepackage{geometry}
\usepackage[utf8]{inputenc} 
\usepackage[T1]{fontenc}    
\usepackage{hyperref}       
\usepackage{url}            
\usepackage{booktabs}       
\usepackage{amsfonts}       
\usepackage{nicefrac}       
\usepackage{microtype}      
\usepackage{amsmath}
\usepackage{amssymb}
\usepackage{amsthm}
\usepackage{bbm}
\usepackage{color}
\usepackage{authblk}
\usepackage{algorithm2e}
\usepackage{algcompatible}

\theoremstyle{plain}
\newtheorem{theorem}{Theorem}
  \theoremstyle{definition}
  
  \theoremstyle{plain}
  \newtheorem{proposition}{Proposition}
  \theoremstyle{plain}
  \newtheorem{lemma}{Lemma}
  \theoremstyle{plain}

\newcommand{\E}{\mathcal{E}}
\newcommand{\G}{\mathcal{G}}
\newcommand{\V}{\mathcal{V}}
\newcommand{\p}{\mathbf{P}}
\newcommand{\x}{\mathbf{x}}
\newcommand{\betahat}{\hat{\beta}}
\newcommand{\kappahat}{\hat{\kappa}}

\newcommand{\Bhat}{\hat{B}}
\newcommand{\Sigmahat}{\hat{\Sigma}}
\newcommand{\Thetahat}{\hat{\Theta}}
\newcommand{\psihat}{\hat{\psi}}
\newcommand{\Var}{\mathbf{Var}}
\newcommand{\Ex}{\mathbf{E}}

\newcommand{\dice}{\textsc{Dice}}
\newcommand{\slice}{\textsc{Slice}}

\DeclareMathOperator*{\argmin}{argmin}

\newcommand{\prob}{\mathbbm{P}}

\title{Information Theoretic Optimal Learning of Gaussian Graphical Models}
\author{Sidhant Misra\thanks{sidhant@lanl.gov} , Marc Vuffray\thanks{vuffray@lanl.gov} , Andrey Y. Lokhov\thanks{lokhov@lanl.gov}}
\affil{Theoretical Division, Los Alamos National Laboratory, Los Alamos, NM 87545, USA}

\usepackage{babel}

\begin{document}
\global\long\def\argmin{\operatornamewithlimits{argmin}}
\global\long\def\argmax{\operatornamewithlimits{argmax}}

\maketitle

\begin{abstract}
What is the \emph{optimal} number of independent observations from which a sparse Gaussian Graphical Model can be correctly recovered? Information-theoretic arguments provide a lower bound on the minimum number of samples necessary to perfectly identify the support of any multivariate normal distribution as a function of model parameters. For a model defined on a sparse graph with $p$ nodes, a maximum degree $d$ and minimum normalized edge strength $\kappa$, this necessary number of samples scales at least as $d \log p/\kappa^2$. The sample complexity requirements of existing methods for perfect graph reconstruction exhibit dependency on additional parameters that do not enter in the lower bound. The question of whether the  lower bound is tight and achievable by a polynomial time algorithm remains open. In this paper, we constructively answer this question and propose an algorithm, termed \dice, whose sample complexity matches the information-theoretic lower bound up to a universal constant factor. We also propose a related algorithm \slice\ that has a slightly higher sample complexity, but can be implemented as a mixed integer quadratic program which makes it attractive in practice. Importantly, \slice\ retains a critical advantage of \dice\ in that its sample complexity only depends on quantities present in the information theoretic lower bound. We anticipate that this result will stimulate future search of computationally efficient sample-optimal algorithms.   
\end{abstract}

\section{Introduction}
Gaussian Graphical Models (GGMs) are powerful modelling tools for representing statistical dependencies between variables in the form of undirected graphs that are widely used throughout a large number of fields, including neuroscience \cite{huang2010learning,NIPS2010_4080}, gene regulatory networks \cite{basso2005reverse, menendez2010gene} and protein interactions \cite{Friedman799, jones2012psicov}. The popularity of GGMs in applications can be explained by the fact that multivariate Normal distribution approximately describes physical variables represented by sums of independent factors and has maximum entropy among all continuous-variable distributions with a given mean and covariance. Moreover, the sparsity pattern of the graph underlying the GGM provides interpretable structural information on the conditional dependencies between variables through the so-called separation property of Markov Random Fields (MRFs).

In this paper, we study the inverse problem of learning a sparse GGM from a small number of observations. Consider a multivariate Normal distribution defined on a graph $\G=(\V,\E)$ with $|\V|=p$ and bounded maximum degree $d$:
\begin{align}
    \prob(\x) = \frac{\sqrt{\text{det}(\Theta)}}{(2\pi)^{\frac{p}{2}}} \exp \left(-\frac{1}{2}\sum_{i \in \V} \Theta_{ii} (x_i - \mu_i)^2 - \sum_{(i,j) \in \E} \Theta_{ij}(x_i-\mu_i)(x_j-\mu_j) \right),\label{eq:GGM_pdf}
    \vspace{-0.1in}
\end{align}
where $\mu_i$ denotes the mean of the variable $x_i$ and $\Theta$ is the \emph{precision matrix} whose support is determined by the sparsity pattern of the graph $\G$. GGMs have a special property that $\Theta$ is equal to the inverse of the covariance matrix $\Sigma$, meaning that $(\Sigma^{-1})_{ij} = 0$ for all $(i,j) \notin \E$. In our reconstruction problem, the data is given as a collection of $n$ independent samples $\{x^k_i\}_{i\in \V}$ indexed by $k=1,\ldots,n$ and drawn from the distribution \eqref{eq:GGM_pdf}. We are interested in finding tractable algorithms that with high probability output an accurate estimate $\hat{\G}$ of the graph $\G$, i.e. $\prob(\hat{\G} = \G)  > 1-\delta$ for a given confidence $\delta > 0$. 

The minimum number of samples $n^{*}$ required for perfect sparse graph reconstruction is given by an information-theoretic (IT) lower bound in \cite{wang2010information} that reads
\begin{align} \label{eq:IT_lower}
    n^* > \max \left\{ \frac{\log {\binom{p-d}{2} }-1}{4 \kappa^2}, \frac{2\left(\log {\binom{p}{d}} -1\right)}{\log \left(1+\frac{d \kappa}{1-\kappa}\right) -\frac{d \kappa}{1+(d-1)\kappa}} \right\},
\end{align}
where the parameter $\kappa$ denotes the minimum normalized edge strength and is defined as
\begin{align}
    \kappa = \min_{(i,j) \in \E} \frac{|\Theta_{ij}|}{\sqrt{\Theta_{ii}\Theta_{jj}}}. \label{eq:kappa_def}
\end{align}

Notice that the IT lower bound \eqref{eq:IT_lower} depends solely on three parameters of the problem: dimension $p$, maximum degree $d$, and minimum edge strength $\kappa$. A weak logarithmic dependence on $p$ indicates that it might be possible to reconstruct $\G$ in the high-dimensional regime, and the inverse square dependence on $\kappa$ is natural because it becomes more difficult to distinguish an edge of low strength from its absence as $\kappa \rightarrow 0$. It remained unknown if the bound \eqref{eq:IT_lower} is tight, i.e., if there exists a \emph{sample-optimal} algorithm, say named $\textsc{Algo}$, with a sample complexity $n_{\textsc{Algo}}$ that does not depend on additional parameters and achieves this bound:
\begin{equation}
    n_{\textsc{Algo}}(p,d,\kappa) = C n^{*}(p,d,\kappa),
    \label{eq:converse_IT}
\end{equation}
where $C$ is a universal constant that does not depend on any problem parameters. If \eqref{eq:converse_IT} is satisfied, we say that $\textsc{Algo}$ provides a \emph{converse} result to the IT lower bound \eqref{eq:IT_lower}.

Numerous algorithms have been suggested to reconstruct sparse GGMs; a non-exhaustive list includes \cite{meinshausen2006high,yuan2007model,cai2011constrained,anandkumar2012high,cai2012estimating,johnson2012high,wang2016precision}. However, the sample complexity analysis of previously proposed methods reveals that none of them is converse to the IT bound \eqref{eq:IT_lower}. For most algorithms, the required number of samples depends on additional parameters of the problem that are not present in \eqref{eq:IT_lower}, often due to the assumptions made in the analysis. The regression-type approach of \cite{meinshausen2006high} for estimating the neighborhood of each vertex based on \textsc{Lasso} \cite{tibshirani1996regression} requires certain \emph{incoherence} properties of the precision matrix, reminiscent of the compressed sensing problem. A variant of the incoherence condition is assumed in the analysis \cite{ravikumar2008model} of the $\ell_1$ regularized log-likelihood estimator, commonly known as \textsc{Graph Lasso} \cite{yuan2007model,d2008first}. The proof for $\ell_1$-based estimators \textsc{Clime} \cite{cai2011constrained} and \textsc{Aclime} \cite{cai2012estimating} require that the eigenvalues of the precision matrix are bounded. The conditional covariance thresholding algorithm \cite{anandkumar2012high} was analyzed only for the class of so-called \emph{walk-summable} models. The analysis of non-convex optimization based methods \cite{johnson2012high,wang2016precision} require bounded eigenvalues of $\Theta$ matrix.

Although the above methods have been shown to successfully exploit sparsity and reconstruct the underlying graph $\G$ perfectly with $O(\log p)$ samples, all of them exhibit dependence on the condition number of the precision matrix among other quantities. In particular, the bound on the condition number of $\Theta$ follows from the the most prevalent assumption in the literature, the so-called Restricted Eigenvalues (RE). Consequently, it is widely believed that the RE condition is in fact \emph{necessary} to enable model reconstruction. Notice, however, that the condition number has no impact on the IT lower bound in \eqref{eq:IT_lower}, as stated above. Consider the following example of a simple precision matrix:
\begin{align}
    \Theta = \left[ \begin{array}{cccc}
        1       & \kappa_0    &    \kappa_0    \\
        \kappa_0  & 1         & 1-\epsilon  \\
        \kappa_0  & 1-\epsilon& 1      \\
        \end{array}
        \right],         \label{eq:counter_example}
\end{align}
where $1 - \epsilon > \kappa_0 > 0$. The minimum normalized edge strength for \eqref{eq:counter_example} is given by $\kappa = \kappa_0$. On the other hand, the condition number is given by $\frac{\lambda_{max}(\Theta)}{\lambda_{min}(\Theta)} \geq \epsilon^{-1}$. This means that if we keep $\kappa_0$ fixed and let $\epsilon \rightarrow 0$, the minimum number of samples $n^*$ according to the IT bound in \eqref{eq:IT_lower} remains fixed whereas the condition number, and hence the sample complexity of existing algorithms diverges.

Several approaches that might appear quite natural in this context surprisingly do not successfully eliminate additional parametric dependencies in their sample complexity.
For example, a natural path is to consider conditional independence testing, since it directly exploits the so-called separation property of graphical models. Along these lines one might attempt methods similar to the SGS and PC algorithms found in Spirtes et al., 2001 \cite{spirtes2000causation} and Kalisch et al., 2007 \cite{kalisch2007estimating}. However, in \cite{van2013ell1}, which is the follow up of \cite{kalisch2007estimating}, the authors explicitly pointed out
that conditional independence testing requires strong faithfulness assumptions in the analysis of the PC algorithm in \cite{kalisch2007estimating}. Even in \cite{van2013ell1}, Maximum-Likelihood with exhaustive search requires bounds on eigenvalues of the covariance matrix; therefore, the sample complexity depends on these eigenvalue bounds.
A different approach to utilizing conditional independence testing is based on convergence of the empirical correlation coefficients, such as the one studied in \cite{anandkumar2012high}. However methods based on this idea do not seem to close the IT lower bound. The example below gives an intuitive explanation why. Consider the precision matrix
\begin{align}
     \Theta = \left[ \begin{array}{cccc}
        1       & \kappa    &    \kappa   & 0 \\
        \kappa  & 1         & 1-\epsilon   & 0\\
        \kappa  & 1-\epsilon& 1      & 0\\
        0 & 0& 0& 1\\
        \end{array}
        \right].      
\end{align}

The correlation coefficient of variable corresponding to rows 1 and 2 conditioned on variable 4 can be computed as $
    \rho_{12\mid 4} = \frac{Cov(X_1,X_2\mid X_4)}{\sqrt{Var(X_1\mid X_4) Var(X_2 \mid X_4)}} 
                    = \frac{\kappa \sqrt{\epsilon}}{\sqrt{(1-\kappa^2)(2-\epsilon)}}.$ This implies that to assert that there exists an edge between 1 and 2 one needs $n=O(1/\epsilon)$ samples, whereas for $\epsilon<\kappa$ there is no dependence of the IT lower bound on $\epsilon$. 
                    
Does that mean that the IT lower bound \eqref{eq:IT_lower} is loose, and the bounded condition number of the precision matrix is indeed a necessary condition for the recovery of sparse GGMs?  In this paper, we answer this question constructively, and propose a multi-stage algorithm, named \emph{Degree-constrained Inverse Covariance Estimator} (\dice). Without any assumptions, we show that \dice~reconstructs the graph $G$ perfectly with high probability $1-\delta$ using $2d + \frac{192}{\kappa^2}d\log p + \frac{64}{\kappa^2}\log\left(\frac{4d}{\delta}\right)$ samples, i.e. provides a converse result \eqref{eq:converse_IT} to the IT lower bound \eqref{eq:IT_lower}. Therefore, \dice~closes the gap to the IT bound, and shows that \eqref{eq:IT_lower} is tight. The computational complexity of \dice~is primarily driven by the iterative support testing step, based on comparison of two neighborhoods, and reads $O(p^{2d+1})$, i.e. is polynomial with respect to $p$ in the setting of sparse graphs where $d=O(1)$.

We also propose a related algorithm termed \emph{Sparse Least-squares Inverse Covariance Estimator} (\slice) which uses a subset of the phases used in \dice. Unlike \dice\ this simpler algorithm allows implementation as a mixed integer  quadratic program (MIQP), enabling the use of modern mixed integer solvers that can be very efficient in practice. As a price for the enhanced computational efficiency, the sample complexity of \slice\ is $d + \frac{32}{\kappa^4}\log \left(\frac{4p^{d+1}}{\delta}\right)$, i.e., roughly a factor $1/\kappa^2$ higher than the IT lower bound (and the sample complexity of \dice). However the sample complexity is still only dependent on the parameters present in the IT lower bound, thus avoiding dependence on additional assumptions such as restricted eigenvalues.

The paper is organized as follows: In Section~\ref{sec:algorithm} we introduce our algorithm \dice\ and its sub-routines, and provide main results of our study. Section~\ref{sec:slice} introduces \slice\ and states the theorem regarding its sample complexity. Rigorous mathematical guarantees on the sample performance of our algorithms are given in Section~\ref{sec:proof}, while Section~\ref{sec:proof_of_lemmas} contains proofs of technical lemmas. We conclude with Section~\ref{sec:conclusions} where we discuss some perspectives and open problems.

\section{DICE: Reconstructing Gaussian Graphical Models with Information Theoretically optimal number of samples}
\label{sec:algorithm}
In this section we provide details of \dice. The three constituent steps are (i) cardinality constrained regression to obtain an estimate of the conditional variance for each variable,  (ii) an iterative support testing method to find a size $d$ neighborhood that contains the right support, and (iii) a clean up phase to eliminate the non-edges in the set found in (ii). For simplicity of notation, we assume that the distribution in consideration has zero mean. All results easily generalize to the non-zero mean case, as stated below.

\subsection{Phase 1: Estimating conditional variances}  \label{subsec:dice_phase1}
The first step of the algorithm obtains an estimate of the conditional variance of each variable $i \in \V$ where the conditioning is for all neighbors of $i$. For each $i \in \V$ our estimate $\Thetahat_{ii}$ of $\Theta_{ii}$  is given by 
\begin{align}
    \frac{1}{\Thetahat_{ii}} = \min_{\betahat \in \mathbbm{R}^{p-1}} \quad & L_{i}(\betahat, \Sigmahat) = \frac{1}{n} \sum_{k=1}^{n} \left(x_i^k + \sum_{j \neq i} \beta_{ij} x_j\right)^2,  \nonumber\\
    \mbox{s.t.} \quad & \|\betahat\|_0 \leq d, \label{eq:l0}
\end{align}
where the $\ell_0$-norm counts the number of non-zero components and $\Sigmahat$ denotes the empirical covariance matrix whose components are given by $\Sigmahat_{ij} = \frac{1}{n} \sum_{k=1}^{n} x_i^k x_j^k$.
\footnote{All results in this paper directly generalize to the case of non-zero mean by replacing the $\Sigmahat$ above by the unbiased covariance estimator given by $\Sigmahat_{ij} = \frac{1}{n-1}\sum_{k=1}^{n}(x_i^k - \bar{x}_i)(x_j^k-\bar{x}_j)$, where $\bar{x}_i = \frac{1}{n}\sum_{k=1}^n x_i^k$.}

Since the $\ell_0$ constraint in \eqref{eq:l0} is equivalent to searching over all possible $\betahat$ with support given by some $A \subset [p]$ with $|A| = d$. The optimization in \eqref{eq:l0} can be re-written as
\begin{align}
 \frac{1}{\Thetahat_{ii}} = \min_{A \subseteq [p] \setminus i\ :\ |A| = d} \quad \min_{\betahat \in \mathbbm{R}^{p-1}: \text{Supp}(\betahat) \subset A} L_i(\betahat,\Sigmahat). \label{eq:cond_var_reform}
\end{align}
Since we will be restricting ourselves to the case when $2d+1 < n$, each $d\times d$ submatrix of $\Sigmahat$ has full rank, and  the inner minimization in \eqref{eq:cond_var_reform} can be explicitly resolved to get
\begin{align}
    \betahat_{iA} = -\Sigmahat_{A A}^{-1} \Sigmahat_{A i}. \label{eq:betahat_def}
\end{align}
The corresponding optimal value is given by
\begin{align}
    L_i^*(A, \Sigmahat) &= L_i(\betahat_{iA},\Sigmahat) = \Sigmahat_{ii} - \Sigmahat_{i A} \Sigmahat_{A A}^{-1} \Sigmahat_{A i} \stackrel{(a)}{=} \left[ \Sigmahat^{-1}_{(iA)(iA)}\right]_{11} \stackrel{(b)}{=} \widehat{\Var}(X_i | X_{A}), \label{eq:optimum} 
\end{align}
where $(a)$ is obtained by using the matrix inversion lemma. We obtain $(b)$ from the standard expression for conditional variance of $X_i$ conditioned on $X_{A}$ in multivariate gaussians, explainig the name of this subsection. Notice that in the limit of large number of samples, the empirical conditional variance is replaced by the true one:
\begin{align}
L_i^*(B_i,\Sigma) = \Var(X_i | X_{B_i}) \quad \forall B_i \subset [p] \setminus \{i\}. \label{eq:cond_var_interp}
\end{align}

\subsection{Phase 2: Iterative Support Testing} \label{subsec:dice_phase2}
In this phase, all candidate neighborhoods are passed through a testing criterion. This phase constitutes the main part of the algorithm \dice. We describe the testing criterion in detail and give intuitive rationale behind it.

Fix $i \in \V$ and consider a candidate neighborhood $B_1 \subset \V \setminus \{ i\}$ with $|B_1| = d$. The goal is to obtain a $B_1$ such that $B_i \subseteq B_1$, where $B_i$ denotes the true neighborhood of $i$. The candidate $B_1$ is tested by using a set of \emph{adversarial neighborhoods} $B_2 \subset \V \setminus \{\{ i\} \cup B_1\}$ with $|B_2| = d$. The testing criterion is based on the regression coefficients $\betahat_{iB_1B_2} = - \Sigmahat_{B_1 B_2,B_1 B_2}^{-1} \Sigmahat_{B_1 B_2,i}$ as in \eqref{eq:betahat_def}, where we use the notation $B_1B_2 = B_1 \cup B_2$ for simplicity. The candidate $B_1$ is deemed to have \textsc{passed} the testing criterion if for all adversarial neighborhoods $B_2$ we have
\begin{align}
   \max_{j \in B_2}  \kappahat_{ij}  := |\betahat_{ij}|\sqrt{\frac{\Thetahat_{ii}}{\Thetahat_{jj}}} < \frac{\kappa}{2}. \label{eq:testing_criterion}
\end{align}
The quantities $\kappahat_{ij}$ can be considered as \emph{estimated} normalized edge strengths.
The testing criterion relies on the fact that when the number of samples is sufficient (which we formalize later in Section~\ref{subsec:main_result}), the quantities $\kappahat_{ij}$ are accurate empirical estimates of the true normalized edge strengths given by $\kappa_{ij} = |\beta_{ij}\beta_{ji}| = \frac{|\Theta_{ij}|}{\sqrt{\Theta_{ii}\Theta_{jj}}}$. The intuitive logic behind the testing criterion in \eqref{eq:testing_criterion} can be explained by considering the following two cases:
\begin{itemize}
    \item \emph{Case 1: The candidate $B_1$ in consideration is such that $B_i \subseteq B_1$, where $B_i$ denotes the true neighborhood of $i$.} In this case, for every adversary $B_2$, and assuming that the estimates $\betahat_{iB_1B_2}$ are accurate enough, for each $j \in B_2$ the estimates $\kappahat_{ij}$ should be close to the true value $\kappa_{ij} = 0$, i.e., $\kappahat_{ij} < \kappa/2$ for all $j \in B_2$ and $B_1$ would \text{pass} the test in \eqref{eq:testing_criterion}.
    \item \emph{Case 2: There exists $j \in B_i \setminus B_1$: } In this case the set $B_1$ has missed a neigbor $j \in B_i$. Here, any adversary $B_2$ such that $B_i \subset B_1 \cup B_2$ will make $B_1$ fail the testing criterion. This is again because, for $j \in B_i \setminus B_1$, the quantity $\kappahat_{ij}$ is expected to be close to its true value $\kappa_{ij} > \kappa$, i.e., $\kappahat_{ij} > \kappa/2$ and hence $B_1$ would \textsc{fail} the test in \eqref{eq:testing_criterion}.
\end{itemize}

\subsection{Phase 3: Eliminate non-edges} \label{subsec:dice_phase3}
Once a set $B_1$ is obtained in Phase $2$ such that $|B_1| = d$ and $B_i \subseteq B_1$, this clean-up phase consists of appending any $B_2$ to $B_1$, computing the estimated normalized couplings $\kappahat_{ij}$ for all $j \in B_1$ and declaring any $j \in B_1$ such that $\kappahat_{ij} < \kappa/2$ as a non-edge. The success of this step also relies on the accuracy of the estimates $\kappahat_{ij}$. The intuitive description of the algorithm and its performance is formalized in the next subsection. 

\subsection{Formal description of the algorithm and main result} \label{subsec:main_result}
In this subsection, we state our main result regarding the sample complexity of \dice, which is formally presented in Algorithm~1.

\begin{algorithm}[!htb]
\SetAlgoLined
\textbf{\emph{Phase 1: Estimatimating conditional variances}} \\ 
\For{$i \in \V$}{
    Estimate $\Thetahat_{ii}$ by solving \eqref{eq:l0}
    }
    
\textbf{\emph{Phase 2: Iterative support testing}} \\ 
\For{$i \in \V$}{
        \For{$B_1 \subset \V \setminus  \{i\}$ s.t. $|B_1| = d$}{
            {\small PASSED} $\leftarrow$  {\small YES} \\
            \For{$B_2 \subset \V \setminus \{B_1 \cup \{i\}\}$ s.t. $|B_2| = d$}{
                Compute $\betahat_{i,B_1 B_2}$ following \eqref{eq:betahat_def}\\
                Estimate $\kappahat_{ij}$ following \eqref{eq:testing_criterion}\\
                \If{$\max_{j \in B_2} \kappahat_{ij} > \kappa/2$}{
                    {\small PASSED} $\leftarrow$ {\small NO} \\
                    {\bf break}
                }
            }
             \If{{\small PASSED} = {\small YES}}{
                    $\tilde{B}_i \leftarrow B_1$ \\
                    {\bf break}
                }
        }
    }
\textbf{\emph{Phase 3: Eliminate non-edges}} \\     
\For{$i \in \V$}{
    Choose any $B_2 \subset \V \setminus \{B_1 \cup \{i\}\}$ s.t. $|B_2| = d$\\
    \For{$j \in \tilde{B}_i$}{
        Compute $\kappahat_{ij}$ following \eqref{eq:testing_criterion}
    }
    $\hat{B}_i \leftarrow \{j \in \tilde{B}_i \mid \kappahat_{ij} > \frac{\kappa}{2}\}$
}
{\bf return} $\hat{B}_i$ for $i \in \V$
\caption{\dice($p$, $d$, $\kappa$)}
\label{alg}
\end{algorithm}

The following is the main result of the paper which proves that the algorithm \dice\  achieves the information theoretic lower bound in \eqref{eq:IT_lower} up to a universal constant.
\begin{theorem}[Converse to IT bound] \label{thm:main}
Given $\delta > 0$, the probability of perfect graph reconstruction using \dice\ is lower bounded as 
\begin{align}
    \prob(\hat{\G} = \G)  > 1-\delta,  
\end{align}
provided that the number of samples satisfies \footnote{In particular, Theorem \ref{thm:main} is valid for a number of samples $n>\frac{320}{\kappa^2}(d\log p + \log(1/\delta))$.} 
\begin{align}
    n > 2d + \frac{192}{\kappa^2}d\log p + \frac{64}{\kappa^2}\log\left(\frac{4d}{\delta}\right).   \label{eq:sample_complexity}
\end{align}
\end{theorem}

\section{SLICE: Reconstructing Gaussian Graphical Models with near optimal number of samples using Mixed Integer Quadratic Programming}  \label{sec:slice}
In this section we state the details of the \slice\ algorithm. \slice\ trades-off some optimality with respect to  sample complexity for better computational complexity and enable implementation using a mixed integer quadratic programming formulation. With the rapid progress in mixed integer programming technology, this offers a significant advantage over the exhaustive search required for \dice\ in terms of practical efficiency. The algorithm \slice\ simply utilizes the Phase 1 (Section~\ref{subsec:dice_phase1}) of \dice\ followed by a variation of the product and threshold procedure in Phase 3 (Section~\ref{subsec:dice_phase3}) of \dice\ in order to eliminate non-edges and estimate the exact support. By skipping the iterative neighborhood testing in Phase 2 (Section~\ref{subsec:dice_phase2}), \slice\ improves upon the computational efficiency of \dice, but with a penalty of an additional $\frac{1}{\kappa^2}$ factor in the required number of samples. The various phases of \slice\ are described in the following sections.

\subsection{Phase 1: Least Squares with $\ell_0$-constraint}  \label{subsec:slice_phase1}
The first step of the algorithm is identical to that of \dice, but the purpose is different. While for \dice, the only purpose was to estimate the conditional variances, \slice\ requires the estimates of the regression coefficients:
\begin{align}
    \betahat_i = \argmin_{\beta_i \in \mathbbm{R}^{p-1}} \quad & L_{i}(\beta_i, \Sigmahat) = \frac{1}{n} \sum_{k=1}^{n} \left(x_i^k + \sum_{j \neq i} \beta_{ij} x_j\right)^2,  \nonumber\\
    \mbox{s.t.} \quad & \|\beta_i\|_0 \leq d, \label{eq:obj}
\end{align}

\subsection{Phase 2: Estimate the support} \label{subsec:slice_phase2}
Once the estimates $\betahat_i$ have been obtained for all $i \in \V$, we estimate the edge-set $\hat{\E}$ through the following thresholding procedure
 \begin{align}
     \hat{\E} = \left\{(i,j)\in \V \times \V : \sqrt{|\betahat_{ij} \times \betahat_{ji}|} > \kappa/2\right\}.
 \end{align}
 The estimated graph is then declared as  $\hat{\G} = (\V, \hat{\E})$.

\subsection{Implementation as a mixed integer quadratic program}
Phase $1$ of the \slice\ algorithm has a computational complexity of $O(p^{d+1})$ since it is equivalent to an exhaustive search over all possible size $d$ neighborhood of each vertex $i \in \V$. The second step can be implemented with a much lower computational complexity of $O(pd)$ leading to an overall complexity of $O(p^{d+1})$. 

However when $d$ is not small enough, performing an exhaustive search can be prohibitively expensive. Instead, the problem can be reformulated as a Mixed Integer Quadratic Program (MIQP), which in practice is significantly faster, especially when using modern mixed integer solvers such as CPLEX or GUROBI. In the context of compressive sensing and sparse regression, the use of MIQP has been explored in \cite{bertsimas2016best} to solve a $\ell_0$ constrained quadratic objective. We present one such formulation:
\begin{subequations}
\begin{align}
\min_{\beta_i \in \mathbbm{R}^{p-1}} \quad &  \beta_i^T \Sigmahat_{\bar{i}\bar{i}} \beta_i +  2 \Sigmahat_{i\bar{i}} \beta_i +  \Sigmahat_{ii} \\
\mbox{s.t.} \quad & s_{ij} L \leq \beta_{ij} \leq s_{ij} U, \quad \forall j \neq i \\
& \sum_{j \neq i} s_{ij} = d, \\
& s_{ij} \in \{0,1\}, \quad \forall j \neq i.
\end{align}
\end{subequations}
In the above $L$ and $U$ denote known or estimated upper and lower bounds on the regression variables. For a more detailed discussion on obtaining these bounds, and formulations that avoid them, we refer the reader to \cite{bertsimas2016best}.

\subsection{Sample complexity of \slice} \label{subsec:slice_main_result}
In this subsection, we state the theoretical result regarding the sample complexity of \slice.
\begin{theorem}[Sample complexity of \slice] \label{thm:slice}
Given $\delta > 0$, the probability of perfect graph reconstruction using \textsc{Slice} is lower bounded as $ \p(\hat{\G} = \G)  > 1-\delta,$ provided that the number of samples satisfies
\begin{align}
    n > d + \frac{32}{\kappa^4}\log \left(\frac{4p^{d+1}}{\delta}\right).
\end{align}
\end{theorem}
\slice\ retains a critical advantage of \dice, which is its insensitivity to parameters absent in the IT lower bound \ref{eq:IT_lower}. In the next subsection, we demonstrate this advantage of \slice\ through some illustrative numerical examples.

\subsection{Numerical illustration of condition number independence of \slice}

In this section, we construct a very simple counterexample, consisting of a sequence of matrices with growing condition number $\frac{\lambda_{max}}{\lambda_{min}}$ but fixed minimum normalized edge strength $\kappa$.
The primary purpose of this experiment is to demonstrate that the sample complexity of existing reconstruction algorithms are indeed sensitive to the condition number of the precision matrix $\Theta$, whereas the sample complexity dictated by the IT lower bound in \eqref{eq:IT_lower} as well as our proposed algorithm  \textsc{Slice} shows no such dependence.

The counter example sequence inspired by \eqref{eq:counter_example} consists of a triangle with two weak links and one stronger link and a collection of independent nodes. This family of GGMs is parametrized by the following inverse covariance matrix,
\begin{align}
\Theta_{\kappa,\epsilon,\sigma}=\left[ \begin{array}{cccc}
        1       & \kappa    &    \kappa & 0 \\
        \kappa  & 1         & 1-\epsilon& 0 \\
        \kappa  & 1-\epsilon& 1         & 0 \\
        0       & 0         & 0         & \hspace{-0.25cm} \frac{1}{\sigma^2} I_{(p-3)\times(p-3)} \\
        \end{array}
        \right], \quad
        \label{eq:matrix}
\end{align}
where $1-\epsilon$ is the strength of the strong link, $\kappa < 1-\epsilon$ is the strength of the weak links and $\sigma^2$ is the variance of the independent nodes. This family of graphs are chosen such that $\kappa$ in \eqref{eq:matrix} corresponds to the minimum normalized edge strength in \eqref{eq:kappa_def}. Note that the maximum degree is $d=2$. This problem can be interpreted as detecting a triangle within a cloud of independent nodes, a situation that is very plausible in practice. 

The simulations are performed for matrix dimension $p=200$ and $n=175$ samples which satisfies $n<p$. We repeat the reconstruction procedure 50 times with independent samples for different values of $\sigma^2 \in \{1,\ldots,10^{4}\}$ while $\kappa=0.4$ and $\epsilon=0.01$ are fixed. The regularizer parameters in \textsc{Aclime}, \textsc{Lasso} and \textsc{Graph Lasso} have been optimized to yield the best possible results for each value of $\sigma^2$, an advantage that cannot be availed in practice. \textsc{Slice} inherently does not have this issue.

For each algorithm we compute its estimate $\hat{\kappa}_{12}$ and $\hat{\kappa}_{14}$ of the normalized link values (1,2) and (1,4),
\begin{align}
    \kappa_{12} = \sqrt{\frac{\Theta_{12} \Theta_{21}}{\Theta_{11} \Theta_{22}}}, \quad \kappa_{14} = \sqrt{\frac{\Theta_{14} \Theta_{41}}{\Theta_{11} \Theta_{44}}}.\label{eq:link_strength}
\end{align}

We declare that an algorithm fails to reconstruct the graph whenever $\hat{\kappa}_{12} \leq \hat{\kappa}_{14}$: if this condition is satisfied, then links (1,2) or/and (1,4) are incorrectly reconstructed regardless of the thresholding procedure. Note that this choice of reconstruction failure criterion is quite generous. It is very unlikely that one can devise a successful thresholding procedure solely based on the criterion $\hat{\kappa}_{12} \leq \hat{\kappa}_{14}$ when $\hat{\kappa}_{12}$ and $\hat{\kappa}_{14}$ are close to each other. 
This is particularly true for several reconstructions provided by \textsc{Graph Lasso} and \textsc{Aclime} as illustrated in Figure~\ref{fig:scatter}, whereas the procedure appears to provide no advantage to \textsc{SLICE}.
Note that we compare normalized link strengths $\kappa_{ij}$, which are invariant to rescaling of the $\Theta$ matrix, instead of matrix element ratios $\beta_{ij}$ or matrix elements $\Theta_{ij}$. Thus reconstruction based on the latter quantities would fail for some rescaling of $\Theta$.

\begin{figure}[!bth]
\begin{center}
\raisebox{0.00\height}{\includegraphics[width=0.8\columnwidth]{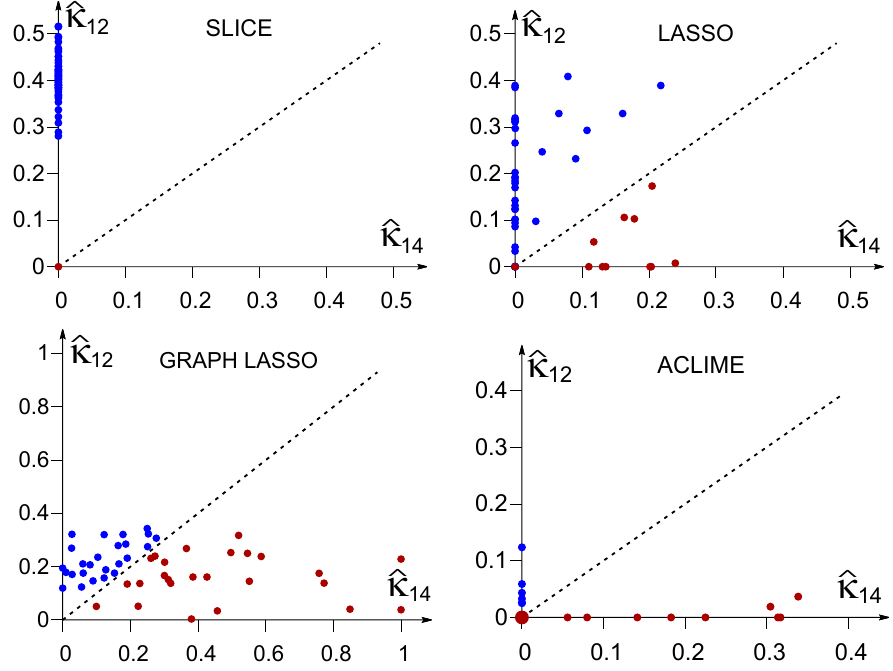}}
\vskip -0.1in
\caption{{\bf Illustration of reconstruction failure for $\sigma^2=\sqrt{1000}$} We show the scatter plot of reconstructed values $\hat{\kappa}_{12}$ and $\hat{\kappa}_{14}$ obtained through $50$ trial reconstructions. We declare that reconstruction fails when $\hat{\kappa}_{12}<\hat{\kappa}_{14}$ which corresponds to points in the lower right part of the graphs (highlighted in red). \textsc{Slice} demonstrates an almost ideal behavior with $\hat{\kappa}_{14}=0$ and $\hat{\kappa}_{12}>\kappa/2=0.2$. Note that \textsc{Graph Lasso} and \textsc{Aclime} systematically yield $\hat{\kappa}_{14}>\hat{\kappa}_{12}$ although (1,4) is not an existing edge. }
\label{fig:scatter}
\end{center}
\vskip -0.1in
\end{figure}

Simulation results are summarized in Figure~\ref{fig:p_failure} where the probability of failure is plotted against the variance $\sigma^2$ of the independent nodes. For $\sigma^2$ close to one, all four algorithms succeed with high-probability and are able to correctly identify that there is a link $(1,2)$ and no link between $(1,4)$. However for larger value of $\sigma^2$, the probability of failure of \textsc{Aclime}, \textsc{Lasso} and \textsc{Graph Lasso} is close to one while \textsc{Slice} remains insensitive to changes in $\sigma^2$. This simple example highlights that when the sample complexity of algorithms depends on parameters not present in the information theoretic bound, the graph reconstruction can be adversarially affected even 
by the presence of additional independent nodes.

\begin{figure}[!th]
\begin{center}
\raisebox{0.00\height}{\includegraphics[width=0.66\columnwidth]{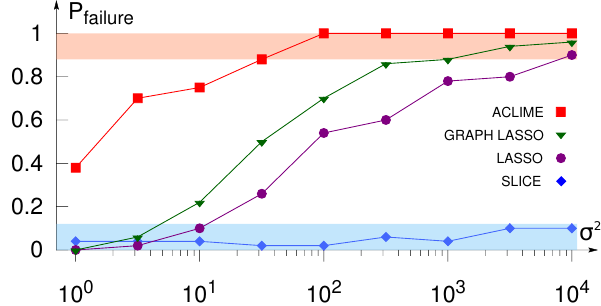}}
\vskip -0.1in
\caption{{\bf Detecting a triangle in a cloud of independent nodes.} Empirical probability of failure averaged over 50 trial reconstructions with a fixed number of samples $n=175$. The values for all algorithms except \textsc{Slice} have been optimized over regularization coefficients. All algorithms except \textsc{SLICE} fail for large $\sigma^2$ values.}
\label{fig:p_failure}
\end{center}
\vskip -0.16in
\end{figure}

In Appendix~\ref{sec:scalability_slice}, we conduct additional numerical studies on synthetic and real data that show that the use of modern Mixed-Integer Quadratic Programming solvers allows one to scale up \slice\ even to relatively large problems.

\section{Proof of main results} \label{sec:proof}
In this section, we prove Theorem~\ref{thm:main} and Theorem~\ref{thm:slice}, along with essential propositions. 

\subsection{Proof of Theorem~\ref{thm:main}}
As described intuitively in Section~\ref{sec:algorithm}, the success of \dice\  relies on the fact that the estimates $\kappahat_{ij}$ in \eqref{eq:testing_criterion} are accurate. This fact is established in the following two propositions.
\begin{proposition}[Accuracy of $\Thetahat_{ii}$] \label{prop:l0_estimates}
   Given $\epsilon > 0$, the diagonal entries reconstructed by \eqref{eq:l0} satisfies 
\begin{align}
    \frac{\Theta_{ii}}{1+\epsilon} \leq \Thetahat_{ii} \leq \frac{\Theta_{ii}}{1-\epsilon}, \quad \forall i \in \V ,  \label{eq:l0_estimates}
\end{align}
with probability at least $1 - \delta_1$ provided that the number of samples satisfies
\begin{align}
    n > d + \frac{8}{\epsilon^2} d\log p  +  \frac{8}{\epsilon^2}\log \left(\frac{2d}{\delta_1} \right). \label{eq:n_for_l0}
\end{align}
\end{proposition}

\begin{proposition}[Accuracy of $\betahat_{ij}$] \label{prop:betahat_estimates}
    Given $\epsilon>0$, the regression coefficients $\betahat_{ij}$ satisfy
    \begin{align}
        \left|\betahat_{ij} - \frac{\Theta_{ij}}{\Theta_{ii}} \right| \leq \epsilon \sqrt{\frac{\Theta_{jj}}{\Theta_{ii}}}, \quad \forall j \in A,\quad  \forall A \subset \V \setminus i \quad \mbox{s.t.}\quad  B_i \subset A, \quad  |A| = 2d,  \label{eq:betahat_estimates}
    \end{align}
    with probability at least $1 - \delta_2$, provided that the number of samples satisfies
    \begin{align}
        n > 2d + \frac{8}{\epsilon^2}d\log p + \frac{4}{\epsilon^2}\log\left(\frac{2d}{\delta_2} \right). \label{eq:n_for_betahat}
    \end{align}
    The quantities $\betahat_{ij}$ are computed as in \eqref{eq:betahat_def}.
\end{proposition}

\noindent We first show that Propositions~\ref{prop:l0_estimates} and \ref{prop:betahat_estimates} are sufficient to prove Theorem~\ref{thm:main}. \\

\begin{proof}[Proof of Theorem~\ref{thm:main}]
   By using $\epsilon = \kappa/4$ and $\delta_1=\delta_2=\delta/2$ in Proposition~\ref{prop:l0_estimates} and Proposition~\ref{prop:betahat_estimates}, we get using the union bound and the lower bound on the number of samples $n$ in \eqref{eq:sample_complexity}, that the statements in \eqref{eq:l0_estimates} and \eqref{eq:betahat_estimates} hold with probability at least $1-\delta$.  The proof proceeds by examining all three phases of \dice. 
   
   \paragraph{Phase 1:} Using $\epsilon = \kappa/4$ and \eqref{eq:l0_estimates} in Proposition~\ref{prop:l0_estimates}, we have that 
   \begin{align}
        \frac{2}{3} \stackrel{(a)}{<} \sqrt{\frac{4-\kappa}{4+\kappa}} \leq \sqrt{\frac{\Thetahat_{ii}}{\Thetahat_{jj}} \frac{\Theta_{jj}}{\Theta_{ii}}} \leq  \sqrt{\frac{4+\kappa}{4-\kappa}} \stackrel{(b)}{<} 2, \quad \forall i,j \in \V,    \label{eq:thetahat_ratio}   
   \end{align}
   where $(a)$ and $(b)$ follow by using $\kappa \leq 1$.
    \paragraph{Phase 2:} To analyze the performance of this phase, we consider the two cases alluded to in Section~\ref{sec:algorithm} for each candidate neighborhood $B_1$ in the outer loop of phase 2 in \dice.
\paragraph{Case 1: $B_1 \subseteq B_i.$} 
Since $\Theta_{ij} = 0$, for any $B_2$ in the inner for loop, we get using Proposition~\ref{prop:betahat_estimates} that for all $j \in B_2$
\begin{align}
    |\betahat_{ij}| \leq \frac{\kappa}{4} \sqrt{\frac{\Theta_{jj}}{\Theta_{ii}}} \quad \mbox{since } \Theta_{ij} = 0.
\end{align}
Combining with \eqref{eq:thetahat_ratio}, we get 
\begin{align}
    \kappahat_{ij} &= |\betahat_{ij}|\sqrt{\frac{\Thetahat_{ii}}{\Thetahat_{jj}}} < \frac{\kappa}{4} \times 2 = \frac{\kappa}{2}. \label{eq:kappahat_upper}
\end{align}
Therefore for the candidate $B_1$, the inner loop in \dice\ will terminate with \textsc{passed} = \textsc{yes}.

\paragraph{Case 2: $B_1 \nsubseteq B_i.$} 
In this case, there exists $j \in B_i \setminus B_1$. Consider the case when in the inner loop we have $B_2$ such that $B_i \subset B_1 \cup B_2$. Then $j \in B_2$ and $j \notin B_1$. Repeating the previous calculation we have
\begin{align}
    \kappahat_{ij} = |\betahat_{ij}|\sqrt{\frac{\Thetahat_{ii}}{\Thetahat_{jj}}} > \left(\kappa -  \frac{\kappa}{4}\right) \times \frac{2}{3} = \frac{\kappa}{2}. \label{eq:kappahat_lower}
\end{align}
Therefore for the candidate $B_1$, the inner loop in \dice\ will terminate with \textsc{passed} = \textsc{no}.

\paragraph{Phase 3:}
By appending to $B_1$, any $B_2 \in \V \setminus \{\{i\} \cup B_1 \}$ with $|B_2|=d$, and using the computations in \eqref{eq:kappahat_upper} and \eqref{eq:kappahat_lower}, we get that for all $j \in B_1$,
\begin{align}
    \kappahat_{ij} &> \kappa/2 \quad \mbox{if} \quad j \in B_i, \\
    \kappahat_{ij} &< \kappa/2 \quad \mbox{if} \quad j \notin B_i,
\end{align}
and the proof is complete.
\end{proof}

\subsection{Proof of Proposition~\ref{prop:l0_estimates}}
To prove Proposition~\ref{prop:l0_estimates}, we make use of the following lemma regarding the statistical fluctuations of the various conditional variances involved in \eqref{eq:l0}.
\begin{lemma}[Large deviations on $L_i^*(.,\Sigmahat)$] \label{lem:perturbation}
    Let $0<\epsilon<1$ be given. Then for every $i \in \V$ and every subset $A \subset [p]\setminus \{i\}$ with $|A| = d$, we have 
    \begin{align}
        (1-\epsilon) L_i^*(A,\Sigma) \leq L_i^*(A,\Sigmahat) \leq (1+\epsilon) L_i^*(A,\Sigma),
    \end{align}
    with probability at least $1 - 2 p {\binom{p-1}{d}} e^{-(n-d)\epsilon^2/8}$. Here $L_i^*(.,.)$ is defined as in \eqref{eq:optimum}.
\end{lemma}
The proof of the above lemma is deferred to Appendix~\ref{sec:proof_of_lemmas}. We now show that Proposition~\ref{prop:l0_estimates} follows from Lemma~\ref{lem:perturbation}.
\begin{proof}[Proof of Proposition~\ref{prop:l0_estimates}]
    Fix $i$ and consider a subset $A \subseteq \V \setminus \{i\}$ and $|A|=d$ such that $B_i \subseteq A$. Since $B_i \subseteq A$, we have that $L_{i}^{\ast}(A,\Sigma) = 1/\Theta_{ii}$. Further, using Lemma~\ref{lem:perturbation} we get that 
    \begin{align}
         L_{i}^{\ast}(A,\Sigmahat) \leq (1+\epsilon)L_{i}^{\ast}(A,\Sigma) = \frac{1+\epsilon}{\Theta_{ii}}.
    \end{align}
    We consider the reformulation of \eqref{eq:l0} as given in \eqref{eq:cond_var_reform}. Since $A$ is feasible for \eqref{eq:cond_var_reform}, we must have
    \begin{align}
        \frac{1}{\Thetahat_{ii}} \leq \frac{1+\epsilon}{\Theta_{ii}}. \label{eq:thetahat_upper}
    \end{align}

    \noindent For all $A \subseteq \V \setminus \{i\}$ with $|A| = d$, we have
    \begin{align}
         L_{i}^{\ast}(A,\Sigma) = \Var\left(X_{i} \mid X_{A}\right) \stackrel{(a)}{\geq} \Var\left(X_i \mid X_{[p] \setminus \{i\}} \right) \stackrel{(b)}{=} \Var \left( X_i \mid X_{B_i} \right) = \frac{1}{\Theta_{ii}},
    \end{align}
    where $(a)$ follows from the well-known property of multivariate gaussians that conditioning reduces variance, and $(b)$ follows from the so-called \emph{separation property} of graphical models. Using Lemma~\ref{lem:perturbation}, this shows that 
    \begin{align} 
        \frac{1}{\Thetahat_{ii}} = \min_{A \subseteq [p] \setminus i\ :\ |A| = d} L_{i}^{\ast}(A,\Sigmahat) \geq  \min_{A \subseteq [p] \setminus i\ :\ |A| = d}(1-\epsilon)L_{i}^{\ast}(A,\Sigma) \geq \frac{1-\epsilon}{\Theta_{ii}}.  \label{eq:thetahat_lower}
    \end{align}
    The proof follows by combining \eqref{eq:thetahat_upper} and \eqref{eq:thetahat_lower}.
\end{proof}

\subsection{Proof of Proposition~\ref{prop:betahat_estimates}}  \label{subsec:proof_dice_props}
We first state the essential technical lemmas that form the ingredients of the proof.
 \begin{lemma} \label{lem:conditional_distribution}
    Fix $i \in \V$ and $A \subseteq \V \{i\}$ such that $B_i \subseteq A$ and $|A| = 2d$. The conditional distribution of $\betahat_{ij}$ for any $j \in A$ is given by
    \begin{align}
         \betahat_{ij} \mid \Sigmahat_{A A} &\sim \mathcal{N}\left(\frac{\Theta_{ij}}{\Theta_{ii}},  \Theta_{ii}^{-1}\left( \Sigmahat_{A A}^{-1}\right)_{jj}\right),
    \end{align}
    where $\mathcal{N}(.,.)$ denotes the normal distribution.
\end{lemma}

\begin{lemma} \label{lem:variance_bound}
    Fix $i \in \V$ and $A \subseteq \V \setminus \{i\}$ such that $B_i \subseteq A$ and $|A| = 2d$. Then for any $\epsilon > 0$ the random variable $\left( \Sigmahat_{A A}^{-1}\right)_{jj}$ satisfies the following  inequality
    \begin{align}
        \prob & \left( \left[ \Sigmahat_{A A}^{-1}\right]_{jj} > (1-\epsilon)^{-1} \Theta_{jj} \right) \leq e^{-\frac{(n-2d+1)\epsilon^2}{8}}.
    \end{align}
\end{lemma}
These lemmas are proved in Appendix~\ref{sec:proof_of_lemmas}.

\begin{proof}[Proof of Proposition~\ref{prop:betahat_estimates}]
    Define the event $E = \left[ \Sigmahat_{\Bhat_i  \Bhat_i}^{-1}\right]_{jj} \leq (1-\epsilon_1)^{-1} \Theta_{jj}$.
          We bound the deviation of $\betahat_{ij}$ from $\frac{\Theta_{ij}}{\Theta_{ii}}$ as
     \begin{align}
         \prob \left( |\betahat_{ij} - \frac{\Theta_{ij}}{\Theta_{ii}}| \geq \epsilon \sqrt{\frac{\Theta_{jj}}{\Theta_{ii}}} \right) &=  \prob\left( |\betahat_{ij} - \frac{\Theta_{ij}}{\Theta_{ii}}| \geq \epsilon \sqrt{\frac{\Theta_{jj}}{\Theta_{ii}}} \mid E \right) \prob(E) \nonumber \\
         &+  \prob\left( |\betahat_{ij} - \frac{\Theta_{ij}}{\Theta_{ii}}| \geq \epsilon \sqrt{\frac{\Theta_{jj}}{\Theta_{ii}}} \mid E^c \right) \prob(E^c)  \nonumber \\
         & \stackrel{(a)}{\leq}  2\Phi^c\left(\epsilon \sqrt{1- \epsilon_1 }\sqrt{n}  \right) + e^{-(n-2d+1)\epsilon_1^2/8},\label{eq:betahat_ldp_interim}
     \end{align}
    where $(a)$ follows by bounding the first term using Lemma~\ref{lem:conditional_distribution} and  the definition of $E$, and bounding the second term by the probability of the event $E$ using Lemma~\ref{lem:variance_bound}. 
    Setting $\epsilon_1=2\epsilon$, we get
    \begin{align*}
        \prob \left( |\betahat_{ij} - \frac{\Theta_{ij}}{\Theta_{ii}}| \geq \epsilon \sqrt{\frac{\Theta_{jj}}{\Theta_{ii}}} \right) & \leq 2\Phi^c\left(\epsilon \sqrt{1- 2\epsilon }\sqrt{n}  \right) + e^{-(n-2d+1)\epsilon^2/2} \\
        & \leq \frac{2}{\sqrt{2\pi}} \frac{e^{-\epsilon^2(1-2\epsilon)n/2}}{\epsilon \sqrt{1- 2\epsilon }\sqrt{n} } + e^{-(n-2d+1)\epsilon^2/2} \\
        & \stackrel{(a)}{\leq} e^{-(\epsilon^2n/4)}  + e^{-(n-2d+1)\epsilon^2/2} \\
        & \leq 2e^{-(n-2d+1)\epsilon^2/4},
    \end{align*}
    where $(a)$ follows by using $\epsilon \leq 1/4$ and $n > \frac{2}{\epsilon^2}$. Using the union bound, we have that for all $i \in \V$ and all $A \subset \V \setminus i$ such that  $B_i \subset A$ and $|A| = 2d$,
    \begin{align}
        \prob \left( \left|\betahat_{ij} - \frac{\Theta_{ij}}{\Theta_{ii}} \right| \leq \epsilon \sqrt{\frac{\Theta_{jj}}{\Theta_{ii}}} \right) &\leq 2pd\binom{p-1}{2d}e^{-(n-2d+1)\epsilon^2/4} \leq \delta_2,
    \end{align}
    where the last inequality follows from the assumption on $n$ given in \eqref{eq:n_for_betahat}.
\end{proof}

\subsection{Proof of Theorem~\ref{thm:slice}}
    We prove Theorem~\ref{thm:slice} through the results below that provide guarantees for each step of the \slice\ estimator. 

    \begin{proposition}[Optimal support contains the true support] \label{lem:superset}
        For each $i \in \V$, let $\Bhat_i \subset [p]$ be the support of the optimal solution in \eqref{eq:obj} and let $B_i \subset [p]$ be the neighbors of $i$. Then for any $\delta>0$, the support $\Bhat_i$ satisfies $B_i \subseteq \Bhat_i$ with probability greater than $1-\delta/2$, provided that the number of samples satisfies
        \begin{align}
            n-d > \frac{32}{\kappa^4}\log \left(\frac{4p^{d+1}}{\delta}\right).
        \end{align}
    \end{proposition}
    
    \begin{proposition}[Post-processing Proposition] \label{lem:pp}
        Assume that 
        \begin{align}
            n-d > \frac{64}{\kappa^2} \log \left( \frac{8dp}{\delta} \right).
        \end{align}
        Then with probability greater than $1-\delta/2$, the post processing procedure consisting of \textit{Product and Threshold} terminates with exactly the correct support. 
    \end{proposition}
    
    \begin{proof}[Proof of Theorem~\ref{thm:slice}]
         The result follows by combining Proposition~\ref{lem:superset} and Proposition~\ref{lem:pp} and applying the union bound. 
    \end{proof}
    
  \noindent The two propositions above can be proved by reusing  Lemmas~\ref{lem:perturbation},\ref{lem:conditional_distribution},\ref{lem:variance_bound} in Section~\ref{subsec:proof_dice_props} and the following lemma, the proof of which is provided in Appendix~\ref{sec:proof_of_lemmas}.
  
  \begin{lemma}[Multiplicative gap in noiseless optimal solutions] \label{lem:multiplicative_gap}
    Fix $i \in \V$ and let $B_i \subset [p]$ be the neighbors of $i$. Let $\Bhat \subset [p]$ be any subset such that $|\Bhat| = d$ and $B_i \not\subseteq \Bhat$. Then 
    \begin{align}
    L_i^*(\Bhat,\Sigma)  \geq L_i^*(B_i,\Sigma)(1-\kappa^2)^{-1}.
    \end{align}
    \end{lemma}

    \begin{proof}[Proof of Proposition~\ref{lem:superset}]
        Combining Lemma~\ref{lem:multiplicative_gap} and Lemma~\ref{lem:perturbation} and using $\epsilon = \kappa^2/2$ we have for any $i \in \V$ that the sequence of inequalities
        \begin{align}
            L_i^*(B_i,\Sigmahat) &< (1+\epsilon)L_i^*(B_i,\Sigma) < (1+\epsilon)(1-\kappa^2)L_i^*(\Bhat,\Sigma)    \nonumber \\
            &< \frac{(1-\kappa^2)(1+\epsilon)}{1-\epsilon}L_i^*(\Bhat_i,\Sigmahat) < L_i^*(\Bhat,\Sigmahat), \nonumber
        \end{align}
        is satisfied for all $\{\Bhat \subset [p]\setminus \{i\} : |\Bhat|=d, \ B_i \not\subseteq \Bhat \}$ with probability at least 
        $1 - 2 p{\binom{p-1}{d}} e^{-(n-d)\kappa^4/32}$. Therefore
        \begin{align*}
        \p\left(\exists i \in [p] : B_i \not\subseteq \Bhat_i  \right) < 2p {\binom{p-1}{d}} e^{-(n-d)\kappa^4/32} < \delta/2,
        \end{align*}
        where the last inequality follows from $n-d > \frac{32}{\kappa^4}\log \left(\frac{4p^{d+1}}{\delta}\right)$.
    \end{proof}   
    
    \begin{proof}[Proof of Proposition~\ref{lem:pp}]
        Similar to the proof of Proposition~\ref{prop:betahat_estimates}, using Lemma~\ref{lem:conditional_distribution},\ref{lem:variance_bound} we get that for all $i \in \V$ and all $A \subset \V \setminus i$ such that  $B_i \subset A$ and $|A| = d$,
    \begin{align}
        \prob \left( \left|\betahat_{ij} - \frac{\Theta_{ij}}{\Theta_{ii}} \right| \leq \epsilon \sqrt{\frac{\Theta_{jj}}{\Theta_{ii}}} \right) &\leq 2pd\binom{p-1}{d}e^{-(n-d+1)\epsilon^2/4},
    \end{align}
     Using $\epsilon = \kappa/4$ we get
        \begin{align}
            \p &\left( |\betahat_{ij} - \frac{\theta_{ij}}{\theta_{ii}}| \leq \frac{\kappa}{4} \sqrt{\frac{\theta_{jj}}{\theta_{ii}}} \quad \forall j \in \Bhat_i, i \in \V\right) \geq 1 - 2dp  e^{-(n-d+1)\kappa^2/64} \stackrel{(a)}{\geq} 1 - \frac{\delta}{2}, \label{eq:betahat_ldp_final}
        \end{align}
        where the implication $(a)$ is obtained by using $n-d > \frac{64}{\kappa^2} \log \left( \frac{8dp}{\delta} \right)$ in the premise of Proposition~\ref{lem:pp}.
     
     \noindent Using \eqref{eq:betahat_ldp_final} for both $i$ and $j$ we get with probability greater than least $1 - \frac{\delta}{2}$,
     \begin{align}
        |\betahat_{ij} \betahat_{ji}| &\geq \left(\frac{|\theta_{ij}|}{\theta_{ii}} -  \frac{\kappa}{4} \sqrt{\frac{\theta_{jj}}{\theta_{ii}}}\right)\left(\frac{|\theta_{ij}|}{\theta_{jj}} -  \frac{\kappa}{4} \sqrt{\frac{\theta_{ii}}{\theta_{jj}}}\right) = \left(\frac{|\theta_{ij}|}{\sqrt{\theta_{ii}\theta_{jj}}} -  \frac{\kappa}{4}\right)^2 . \label{eq:beta_lower}
     \end{align}
     From \eqref{eq:beta_lower}, we get that for $(i,j) \in \E$ the estimates satisfy $\sqrt{|\betahat_{ij}||\betahat_{ji}|} \geq 3 \kappa/4 > \kappa/2$. An identical argument can be used to show that
      $\sqrt{|\betahat_{ij}||\betahat_{ji}|} \leq \kappa/4 < \kappa/2$. This proves that the post-processing step recovers the exact support.
        
    \end{proof} 
  
\section{Conclusions} \label{sec:conclusions}

In this paper, we propose the polynomial-time algorithm \dice~that provably recovers the support of sparse Gaussian graphical models with an information-theoretic optimal number of samples. On the theoretical side, this result confirms that the incoherence properties and condition number of the precision matrix are not necessary for the reconstruction task, and that the previously derived information-theoretic bound \cite{wang2010information} is tight. From the algorithmic perspective, reconstruction with the least number of samples is critical when the available data is scarce. Hence, even though the computational time of \dice\ can be large, it might still represent a valuable tool in the applications where the cost of additional data collection is larger than the cost of computations and where we expect the condition number of the precision matrix to be large. We also propose a simplified algorithm called \slice\ with slightly higher sample complexity than \dice\ but with better computational complexity and possibility of implementation as a mixed integer quadratic program, making it attractive in practice. \slice\ also retains a critical advantage of \dice\ in the sense that its sample complexity is also independent of any spurious quantities such as the condition number of the precision matrix.

Since we have now established that learning GGMs with an information-theoretic optimal number of samples given in \eqref{eq:IT_lower} is achievable, the challenge for future work is to design new algorithms that improve the computational complexity of \dice\ and \slice\ while still keeping the sample complexity optimal. From a theoretical point of view, this  constitutes another fundamental open problem -- what is the minimal computational complexity of any algorithm that can achieve the information-theoretic optimal sample complexity?

As suggested by the analysis presented in this paper, any algorithm with optimal sample complexity is likely to be assumption-free. In future work, it would be interesting to see if the ideas behind state-of-the-art assumption-free and computationally efficient algorithms for the reconstruction of discrete graphical models such as \cite{NIPS2016_screening,lokhov2016optimal} could be extended to the case of GGMs.


\appendix
\section{Tests of \slice\ scalability on synthetic and real data} \label{sec:scalability_slice}

In this section, we present several tests on synthetic and real data. Our goal is merely to illustrate that the use of modern Mixed-Integer Quadratic Programming (MIQP) solvers such as Gurobi \cite{gurobi} allows one to run \slice\ in a reasonable time even on relatively large realistic problems.  

As a first test, we run \slice\ on synthetic random graph instances of different degrees ($d=3$ and $d=4$) and sizes ($p=10$, $p=100$ and $p=1000$). The link strengths $\kappa_{ij}$ have been randomly generated in the ranges $[0.2,0.4]$ for $d=3$ and $[0.2,0.3]$ for $d=4$ instances. The family of regular random graphs has been chosen to eliminate potential dependencies on the heterogeneity in the degree distributions. For implementation, we used one possible MIQP formulation presented in the Supplementary Material, and the JuMP framework \cite{dunning2017jump} in julia for running the Gurobi solver. The running times for \slice\ with $n=10^4$ samples for each problem instance are presented in the Table~\ref{tab:running_time}. Notice that the practical scaling of running times is significantly better than what one would expect from the worst-case complexity $O(p^{d+1})$ for the full graph reconstruction.

\begin{table}[thb]
\caption{{\bf Comparison of running times for \slice\ on various regular random graphs with $n=10^4$ samples}: Longest MIQP Gurobi solver time and longest total running time for reconstruction of the neighborhood of one node, and total time for learning the entire graph.}
\label{sample-table}
\begin{center}
\begin{small}
\begin{sc}
\begin{tabular}{lcccr}
\toprule
Graph & Max for 1 node & Max for 1 node & Full problem \\
$(p, d)$ & (Gurobi) & (total) & (total) \\
\midrule
$(10,3)$     & 0.01 sec& 2.7 sec& 7.2 sec\\
$(10,4)$   & 0.03 sec& 2.8 sec& 7.6 sec\\
$(100,3)$  & 0.03 sec& 2.7 sec& 19.8 sec\\
$(100,4)$  & 0.04 sec& 2.8 sec& 21.7 sec\\
$(1000,3)$ & 15.7 sec& 19.3 sec& 18 hours\\
$(1000,4)$ & 92.3 sec& 96 sec& 29.3 hours\\
\bottomrule
\end{tabular}
\end{sc}
\end{small}
\end{center}
\label{tab:running_time}
\vskip -0.16in
\end{table}

\begin{figure}[hbt]
\begin{center}
\raisebox{0.00\height}{\includegraphics[width=0.39\columnwidth]{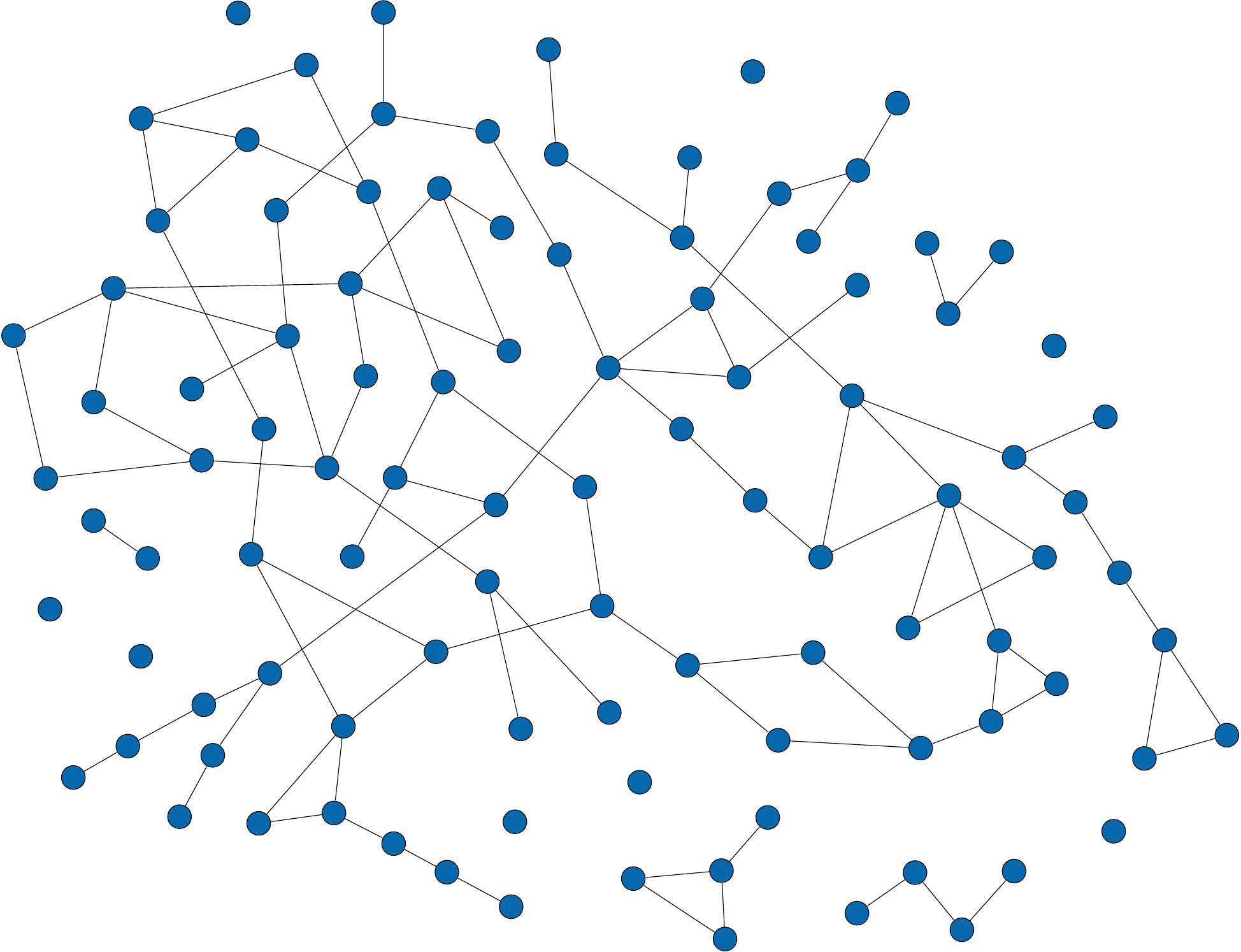}}
\caption{{\bf Graph learned with \slice\ from Riboflavin data set.} This real-world data set \cite{buhlmann2014high} contains $p=101$ variables and $n=71$ samples. In the reconstruction procedure, the maximum degree has been set to $d=6$.}
\label{fig:bio}
\end{center}
\vskip -0.2in
\end{figure}

For the illustration on real data, we use the biological data set related to the Riboflavin production with \emph{B. subtilis}. This data set contains the logarithm of the Riboflavin production rate alongside the logarithms of normalized expression levels of $100$ genes that are most responsive to the Riboflavin production. Hybridization under different fermentation conditions lead to the acquisition of $n = 71$ samples, see \cite{buhlmann2014high} for more details and raw data. The graph reconstructed with \slice\ and constraint $d=6$ is depicted in the Figure~\ref{fig:bio}. It took about $2.5$ days for the algorithm to learn this graph (with the proof of optimality of the obtained solution) in this high-dimensional regime. Notice that again the practical running time for \slice\ using MIQP technology is much lower than the one required to search over the $10^{14}$ candidate neighborhoods of size $d=6$. This example is a perfect illustration of a trade-off between sample and algorithmic complexity in real-world problems where the collection of samples might be very costly.

\section{Proof of technical lemmas} \label{sec:proof_of_lemmas}
We will need the following result from \cite{ouellette1981schur} in the proofs of the technical lemmas. 
\begin{lemma}[\cite{ouellette1981schur} Eq. 6.78] \label{lem:inv_wishart}
    Let $X \in \mathbbm{R}^{k \times k} \sim W(V,l)$ be a random matrix distributed according to the Wishart distribution with parameter $V \succ 0$ and order $l > k-1$. Let $Y = X^{-1}$ with $Y \sim W^{-1}(U, l)$ where $U = V^{-1}$. Let 
            \begin{align*}
                X = \left[ \begin{array}{cc}
                     X_{11} & X_{12}  \\
                     X_{21} & X_{22}
                \end{array}
                \right], \
                V = \left[ \begin{array}{cc}
                     V_{11} & V_{12}  \\
                     V_{21} & V_{22}
                \end{array}
                \right]
            \end{align*}
            be any compatible block matrix representation of $X$ and $V$. Consider block representations for $Y,U$ with the same dimensions $k_1, k_2$ that satisfy $k_1+k_2=k$. Then,
            \begin{itemize}
            \item[(a)] The Schur complements of $X_{11}$ and $Y_{11}$ are distributed as
            \begin{align*}
                X_{11} - X_{12}X_{22}^{-1}X_{21} &\sim  W(V_{11}-V_{12}V_{22}^{-1}V_{21}, l-k_2)  \\
                Y_{11} - Y_{12}Y_{22}^{-1}Y_{21} &= X_{11}^{-1} \sim W^{-1}(V_{11}^{-1}, l)  = W^{-1}(U_{11}-U_{12}U_{22}^{-1}U_{21}, l),
            \end{align*}
            \item[(b)] The random matrix $Y_{22}^{-1}Y_{21}$ conditioned on $X_{11}^{-1}$ is distributed as a matrix normal distribution
            \begin{align*}
                Y_{22}^{-1}Y_{21} \mid X_{11}^{-1} \sim \mathcal{N} \left( U_{22}^{-1}U_{21}, X_{11}^{-1} \otimes U_{22}^{-1} \right).
            \end{align*}
            \end{itemize}
    \end{lemma}

\begin{proof}[Proof of Lemma~\ref{lem:perturbation}]
    Fix $i \in \V$ and $A \subset [p]\setminus \{i\}$ with $|A|  = d$. Then using properties of the Wishart distribution in Lemma~\ref{lem:inv_wishart} part $(a)$, we get that,
    \begin{align}
        \Sigmahat_{ii} - \Sigmahat_{iA} \Sigmahat_{AA}^{-1} \Sigmahat_{Ai} &\sim (\Sigma_{ii} - \Sigma_{i,A} \Sigma_{AA}^{-1} \Sigma_{Ai}) \chi^2_{n-d} \nonumber \\
        &= L_i^*(A,\Sigma) \chi^2_{n-d},
    \end{align}
    where $\chi^2_t$ denotes the standard Chi-squared distribution with $t$ degrees of freedom. Using the Chernoff bound, 
   \begin{align}
       \prob (\chi_{n-d}^2 >1 + \epsilon) &< e^{-(n-d) \left( \frac{\epsilon}{2} - \frac{1}{2}\log(1+\epsilon) \right)} < e^{-(n-d)\epsilon^2/8}, \label{eq:chernoff_upper} \\
       \prob (\chi_{n-d}^2 <1 - \epsilon) &< e^{-(n-d) \left( -\frac{1}{2}\log(1-\epsilon)  - \frac{\epsilon}{2} \right)} < e^{-(n-d)\epsilon^2/8}. \label{eq:chernoff_lower}
   \end{align}
   The proof is completed by using the union bound for all $A \subset [p] \setminus i$ with $|A|=d$ and over all $i \in \V$.
\end{proof}

\begin{proof}[Proof of Lemma~\ref{lem:conditional_distribution}]
    For any $i \in \V$ and $A \subseteq \V \setminus {i}$ with $B_i \subseteq A$ and $|A| = 2d$, let 
    \begin{align}
            \psihat_{(iA)(iA)} = \left(\Sigmahat_{(iA)(iA)}\right)^{-1},  \quad    \Psi_{(iA)(iA)} = \left(\Sigma_{(iA)(iA)}\right)^{-1}.
    \end{align}
    Using the block matrix decomposition for matrix inverse 
         \begin{align}
            \Psi_{(iA)(iA)}  = \Theta_{(iA)(iA)} - \Theta_{(iA) D}\Theta_{DD}^{-1}\Theta_{D (iA)}, \quad \mbox{where}\quad  D = \V \setminus \{i \cup A \}. \label{eq:Psi_schur}
        \end{align}
       Since $B_i \subseteq A$, we must have $\Theta_{iD} = 0$. Hence the matrix $\Psi_{(iA)(iA)}$ satisfies 
        \begin{align}
            \Psi_{ii} = \Theta_{ii}, \quad \Psi_{ij}  &= \Theta_{ij},   \quad \forall j \in A. \label{eq:psi_theta_eq}
        \end{align}
        
         From Lemma~\ref{lem:inv_wishart}, part $(b)$ we get that for all $j \in A$,
         \begin{align}
          \betahat_{ij} \mid \Sigmahat_{A A} &\sim \mathcal{N}\left(\frac{\Psi_{ij}}{\Psi_{ii}},  \Psi_{ii}^{-1}\left( \Sigmahat_{A A}^{-1}\right)_{jj}\right) \\
            &\stackrel{(a)}{=} \mathcal{N}\left(\frac{\Theta_{ij}}{\Theta_{ii}},  \Theta_{ii}^{-1}\left( \Sigmahat_{A A}^{-1}\right)_{jj}\right),
        \end{align}
     where $(a)$ follows from \eqref{eq:psi_theta_eq}.    
\end{proof}

\begin{proof}[Proof of Lemma~\ref{lem:variance_bound}]  
    From \eqref{eq:Psi_schur} we get that for all $j \in A$,
    \begin{align}
        \Psi_{jj} = \Theta_{jj} - \Theta_{jD}\Theta_{DD}^{-1}\Theta_{Dj} \leq \Theta_{jj}. \label{eq:psi_theta_ineq}
    \end{align}
     From Lemma~\ref{lem:inv_wishart}, the random matrix $\Sigmahat_{A A}$ is distributed according to the Wishart distribution  $\Sigmahat_{A A} \sim W\left(\Sigma_{A A},n\right)$. Hence for any $j \in A$,
     \begin{align}
         \left( \left[ \Sigmahat_{A  A}^{-1}\right]_{jj}\right)^{-1} &= \Sigmahat_{jj} - \Sigmahat_{j (A \setminus j)} \Sigmahat_{(A \setminus j)(A \setminus j)}^{-1} \Sigmahat_{(A \setminus j) j} \nonumber \\
         & \stackrel{(a)}{\sim} ( \Sigma_{jj} - \Sigma_{j (A \setminus j)} \Sigma_{(A \setminus j)(A \setminus j)}^{-1} \Sigma_{(A \setminus j) j}) \chi^2_{n-2d+1} \triangleq \alpha_j  \chi^2_{n-d+1}, 
    \end{align}
     where $(a)$ follows from Lemma~\ref{lem:inv_wishart} and we can bound the constant $\alpha_j$ from above for every $j \in A$ as
     \begin{align}
        \alpha_j^{-1} &= \left(\Sigma_{jj} - \Sigma_{j (A \setminus j)} \Sigma_{(A \setminus j)(A \setminus j)}^{-1} \Sigma_{(A \setminus j) j} \right)^{-1} \nonumber = \Psi_{jj} - \Psi_{ji}^2 \Psi_{ii}^{-1} \stackrel{(a)}{\leq} \Psi_{jj} \leq \Theta_{jj}, \label{eq:alpha_bound}
     \end{align}
     where $(a)$ follows from \eqref{eq:psi_theta_ineq}.
   Hence, 
     \begin{align}
         \prob  \left( \left[ \Sigmahat_{A A}^{-1}\right]_{jj} > (1-\epsilon)^{-1} \Theta_{jj} \right)  &= \prob \left( \chi_{n-2d+1}^2 < (1-\epsilon)\alpha_j^{-1}\Theta_{jj}^{-1}  \right) \\
         & \stackrel{(a)}{\leq}  \prob \left( \chi_{n-2d+1}^2 < (1-\epsilon) \right) \\
         & \stackrel{(b)}{\leq} e^{-(n-2d+1)\epsilon^2/8}, \label{eq:var_bound}
     \end{align}
     where $(a)$ follows because by \eqref{eq:alpha_bound} we have $\alpha_j^{-1}\Theta_{jj}^{-1} \leq 1$, and $(b)$ follows from \eqref{eq:chernoff_upper}.
\end{proof}
\begin{proof}[Proof of Lemma~\ref{lem:multiplicative_gap}]
From \eqref{eq:cond_var_interp} we have that 
\begin{align}
L_i^*(B_i,\Sigma) = \Var(X_i | X_{B_i}) \stackrel{(a)}{=} \Var(X_i | X_{[p] \setminus i}) \stackrel{(b)}{=} \theta_{ii}^{-1},
\end{align}
where $(a)$ follows from the separation property of graphical models, and $(b)$ follows from \eqref{eq:optimum}.
Similarly, 
\begin{align}
L_i^*(\Bhat,\Sigma) = \Var(X_i | X_{\Bhat}).
\end{align}
Using the law of total variance we get that 
\begin{align}
\Var&(X_i | X_{\Bhat}) = \Ex\left[ \Var(X_i | X_{B_i \cup \Bhat}) \vert X_{\Bhat} \right]\nonumber \\
&+ \Var\left(  \Ex\left[ X_i | X_{B_i \cup \Bhat}\right] \vert X_{\Bhat} \right)\nonumber \\
&= \frac{1}{\theta_{ii}} + \frac{1}{\theta_{ii}^{2}} \Var\left(\sum_{j \in B_i \cup \Bhat} \theta_{ij} X_{j} \mid X_{\Bhat} \right).
\end{align}
Let $u \in B_i \setminus \Bhat$. From above, we get
\begin{align}
&\Var(X_i|X_{\Bhat}) - \Var(X_i | X_{B_i}) \\
&= \theta_{ii}^{-2} \Var\left(\sum_{j \in B_i \cup \Bhat} \theta_{ij} X_{j} \mid X_{\Bhat} \right) \nonumber \\
& \stackrel{(a)}{\geq} \theta_{ii}^{-2} \Var \left(\sum_{j \in B_i \cup \Bhat} \theta_{ij} X_j \mid X_{[p] \setminus \{i,u\}} \right) \nonumber \\
&= \theta_{ii}^{-1}\left( \frac{\theta_{ii}\theta_{uu}}{\theta_{iu}^2}-1\right)^{-1} \stackrel{(b)}{\geq} \Var(X_i|X_{B_i}) \frac{\kappa^2}{1-\kappa^2}.
\end{align}
The inequality $(a)$ follows from the fact that conditioning reduces variance in Gaussian and observing that $\Bhat \subseteq [p]\setminus \{i,u\}$. The inequality $(b)$ follows from 
\eqref{eq:kappa_def}.
\end{proof}
     
\bibliographystyle{plain}
\bibliography{Gaussian_Learning}



\end{document}